\documentclass[final, 12pt]{colt2019} 

\usepackage{float}

\usepackage{algorithm}
\usepackage{times}
\usepackage{makecell}
\usepackage{multirow}
\usepackage[final]{showlabels}
\allowdisplaybreaks
\usepackage{bbm}
\usepackage{enumerate}

\usepackage{hyperref}       
\usepackage{url}            
\usepackage{amsfonts}       
\usepackage{amsmath}           
\usepackage{textcomp}
\usepackage[]{color-edits}
\usepackage{nicefrac}
\usepackage{xfrac}
\addauthor{HL}{cyan}

\newcommand{\calB}{{\mathcal{B}}}

\newcommand{\calS}{{\mathcal{S}}}

\newcommand{\calK}{{\mathcal{K}}}
\newcommand{\calL}{{\mathcal{L}}}

\newcommand{\reg}{\text{\rm Reg}}

\newcommand{\one}{\boldsymbol{1}}

\newcommand{\e}{\mathbf{e}}

\newcommand{\field}[1]{\mathbb{#1}}

\newcommand{\fR}{\field{R}}

\newcommand{\E}{\field{E}}

\newcommand{\inner}[1]{ \left\langle {#1} \right\rangle }
\newcommand{\inn}[1]{ \langle {#1} \rangle }

\newcommand{\norm}[1]{\left\|{#1}\right\|}

\newcommand{\Ber}{{\text{\rm Ber}}}
\newcommand{\basis}{{\boldsymbol{e}}}

\DeclareMathOperator*{\argmin}{\arg\!\min}

\newcommand{\order}{\ensuremath{\mathcal{O}}}
\newcommand{\otil}{\ensuremath{\widetilde{\mathcal{O}}}}

\newcommand{\LOneBound}{K^{\frac{1}{3}}\sqrt{V_1^{2/3}T_{\phantom{1}}^{1/3}}}
\newcommand{\proj}{{\text{\rm Proj}}}
\newcommand{\dist}{{\text{\rm dist}}}

\newtheorem{assumption}{Assumption}

\usepackage{prettyref}
\newcommand{\pref}[1]{\prettyref{#1}}

\newcommand{\savehyperref}[2]{\texorpdfstring{\hyperref[#1]{#2}}{#2}}
\newrefformat{eq}{\savehyperref{#1}{\textup{(\ref*{#1})}}}
\newrefformat{eqn}{\savehyperref{#1}{Eq.~\eqref{#1}}}
\newrefformat{lem}{\savehyperref{#1}{Lemma~\ref*{#1}}}
\newrefformat{def}{\savehyperref{#1}{Definition~\ref*{#1}}}
\newrefformat{line}{\savehyperref{#1}{line~\ref*{#1}}}
\newrefformat{thm}{\savehyperref{#1}{Theorem~\ref*{#1}}}
\newrefformat{cor}{\savehyperref{#1}{Corollary~\ref*{#1}}}
\newrefformat{sec}{\savehyperref{#1}{Section~\ref*{#1}}}
\newrefformat{app}{\savehyperref{#1}{Appendix~\ref*{#1}}}
\newrefformat{ass}{\savehyperref{#1}{Assumption~\ref*{#1}}}
\newrefformat{ex}{\savehyperref{#1}{Example~\ref*{#1}}}
\newrefformat{fig}{\savehyperref{#1}{Figure~\ref*{#1}}}
\newrefformat{alg}{\savehyperref{#1}{Algorithm~\ref*{#1}}}
\newrefformat{rem}{\savehyperref{#1}{Remark~\ref*{#1}}}
\newrefformat{conj}{\savehyperref{#1}{Conjecture~\ref*{#1}}}
\newrefformat{prop}{\savehyperref{#1}{Proposition~\ref*{#1}}}
\newrefformat{proto}{\savehyperref{#1}{Protocol~\ref*{#1}}}
\newrefformat{prob}{\savehyperref{#1}{Problem~\ref*{#1}}}
\newrefformat{claim}{\savehyperref{#1}{Claim~\ref*{#1}}}
\newrefformat{tab}{\savehyperref{#1}{Table~\ref*{#1}}}

\title[Improved Path-length Regret Bounds for Bandits]{Improved Path-length Regret Bounds for Bandits}
\usepackage{times}

\coltauthor{%
 \Name{S{\'e}bastien Bubeck} \Email{sebubeck@microsoft.com}\\
 \addr Microsoft Research, Redmond
 \AND
 \Name{Yuanzhi Li} \Email{yuanzhil@stanford.edu} \\
 \addr Stanford University  
\AND
 \Name{Haipeng Luo} \Email{haipengl@usc.edu} \\
\addr University of Southern California
\AND
 \Name{Chen-Yu Wei} \Email{chenyu.wei@usc.edu}\\
 \addr University of Southern California
}

\begin{document}

\maketitle

\begin{abstract}
We study adaptive regret bounds in terms of the variation of the losses (the so-called path-length bounds) for both multi-armed bandit and more generally linear bandit.
We first show that the seemingly suboptimal path-length bound of~\citep{wei2018more} is in fact not improvable for adaptive adversary.
Despite this negative result,
we then develop two new algorithms,
one that strictly improves over~\citep{wei2018more} with a smaller path-length measure,
and the other which improves over~\citep{wei2018more} for oblivious adversary when the path-length is large.
Our algorithms are based on the well-studied optimistic mirror descent framework,
but importantly with several novel techniques,
including new optimistic predictions, a slight bias towards recently selected arms,
and the use of a hybrid regularizer similar to that of~\citep{bubeck2018sparsity}.

Furthermore, we extend our results to linear bandit by showing a reduction to obtaining dynamic regret for a full-information problem,
followed by a further reduction to convex body chasing.
As a consequence we obtain new dynamic regret results as well as the first path-length regret bounds for general linear bandit.
\end{abstract}


\section{Introduction}
The multi-armed bandit (MAB) problem~\citep{auer2002nonstochastic} is a classic online learning problem with partial information feedback.
In the general adversarial environment, 
it is well known that $\Theta(\sqrt{KT})$ is the worst-case optimal regret bound
where $T$ is the number of rounds and $K$ is the number of arms.
Linear bandit generalizes MAB to learning linear loss functions with an arbitrary bounded convex set in $\fR^d$,
and it is also known that $\Theta(d\sqrt{T})$ is the worst-case optimal regret~\citep{dani2008price, bubeck2012towards}.

Despite these worst-case bounds,
several works have studied more adaptive algorithms with data-dependent regret bounds that can be much smaller than the worst-case bounds under reasonable conditions.
For example, recent work~\citep{wei2018more} proposes several such data-dependent regret bounds for MAB, including those that replace the dependence on $T$ by the actual losses of the arms, the variance of the losses, or the variation of the losses measured by the so-called path-length.

In particular, since path-length is the smallest among these different measures,
in this work we focus on extending and improving the existing path-length bounds for bandits.
We start from a curious investigation on whether the bound $\otil(\sqrt{KV_1})$\footnote{We use the notation $\otil(\cdot)$ to suppress poly-logarithmic dependence on $T$.} of~\citep{wei2018more} can be improved, where $V_1 = \E[\sum_{t=2}^T\norm{\ell_t - \ell_{t-1}}_1]$ is the 1-norm path-length and $\ell_1, \ldots, \ell_T$ are the loss vectors chosen by the adversary.  
Indeed, since $V_1$ can be as large as $KT$ and $\Omega(\sqrt{KT})$ is a lower bound for MAB, 
it is very natural to ask whether one can improve the bound $\otil(\sqrt{KV_1})$ to $\otil(\sqrt{V_1})$.

Surprisingly, we show (in \pref{thm:L1_norm_lower_bound}) that the bound $\otil(\sqrt{KV_1})$ is not improvable in general, at least not for an {\it adaptive adversary} who can pick the loss vectors based on the learner's previous actions.
Despite this negative result, however, we also show the following two improvements:
\begin{itemize}
\item First, in \pref{sec:max_norm} we propose a new algorithm with regret bound $\otil(\sqrt{KV_\infty})$ where $V_\infty = \E\big[\sum_{t=2}^T\norm{\ell_t - \ell_{t-1}}_\infty \big]$ is the max-norm path-length.
This is a strict improvement over~\citep{wei2018more} since $V_\infty \leq V_1$,
and moreover it is optimal even for oblivious adversary (that is, adversary who picks loss vectors independently of the learner's actions).

\item Second, building on top of our first algorithm, in \pref{sec:L1_norm} we propose a more sophisticated algorithm with regret bound $\otil\Big(\LOneBound\Big)$ for oblivious adversary.
This improves over~\citep{wei2018more} whenever $V_1 \geq T/K$.
For example when $V_1=T$, our bound becomes $\otil\Big(K^{\frac{1}{3}}\sqrt{T}\Big)$ while the one of \citep{wei2018more} becomes the worst-case bound  $\otil(\sqrt{TK})$.
Note that in light of our aforementioned lower bound,
this also shows a strict distinction between oblivious and adaptive adversary,
which is uncommon in online learning.
\end{itemize}

\renewcommand{\arraystretch}{1.8}
\begin{table}[t]
\centering
\caption{Main results and comparisons with previous works
(see \pref{sec:notation} for notation definition).
For linear bandit, the upper bound with $V_2$ holds when the decision set is a 2-norm ball, and the lower bound with $V_1$ holds when the decision set is a max-norm ball.}
\label{tab:results}
\begin{tabular}{|c|c|c|c|}
\hline
& & Oblivious Adversary & Adaptive Adversary \\
\hline
\multirow{2}{*}{MAB} & 
upper bound & 
\multicolumn{2}{c|}{$\sqrt{KV_\infty}$ (\pref{thm:max_norm_upper_bound}) } \\
\cline{2-4}
& 
lower bound & 
\multicolumn{2}{c|}{$\sqrt{KV_\infty}$} \\
\hline
\multirow{2}{*}{MAB} & 
upper bound & 
$\LOneBound$ (\pref{thm:1_norm_upper_bound}) & 
$\sqrt{KV_1}$ \citep{wei2018more} \\
\cline{2-4}
& 
lower bound & 
$\sqrt{V_1}$ & 
$\sqrt{KV_1}$ (\pref{thm:L1_norm_lower_bound}) \\
\hline
\multirow{2}{*}{\makecell{Linear \\ Bandit}} & 
upper bound & 
\multicolumn{2}{c|}{$d^{3/2}\sqrt{V_2}$ (\pref{cor:linear_bandit} or \ref{cor:linear_bandit2}) ; $d^2\sqrt{V_*}$ (\pref{cor:linear_bandit3})}  \\
\cline{2-4}
& 
lower bound& 
\multicolumn{2}{c|}{$d\sqrt{V_1}$ \citep{dani2008price}} \\
\hline
\end{tabular}
\end{table}

Our algorithms are based on the optimistic mirror descent framework~\citep{chiang2012online, rakhlin2013online}.
However, several novel techniques are needed to achieve our results,
including new optimistic predictions, 
a slight bias towards recently selected arms,
and also a hybrid regularizer.
In particular, our second algorithm dynamically partitions the arms into two groups based on their probabilities of being selected,
and applies different optimistic predictions and essentially different regularizers to these two groups.
This new technique might be of independent interest.

Moreover, in \pref{sec:linear_bandit} we further extend our results to general linear bandit and achieve a regret bound $\otil(d^2\sqrt{V_*})$ where $V_* = \E\big[\sum_{t=2}^T\norm{\ell_t - \ell_{t-1}}_* \big]$ is the path-length measured by some arbitrary dual norm.
When the decision set of the learner is a 2-norm ball, we also obtain an improved bound of order $\otil(d^\frac{3}{2}\sqrt{V_2})$ where $V_2 = \E\big[\sum_{t=2}^T\norm{\ell_t - \ell_{t-1}}_2 \big]$.
Our algorithm is based on optimistic SCRiBLe~\citep{abernethy2008competing, hazan2011better, rakhlin2013online} and the key challenge is to come up with the optimistic prediction under partial information feedback.
We reduce this problem to obtaining dynamic regret~\citep{zinkevich2003online}  for a full-information online learning problem, and then further reduce the latter to an instance of convex body chasing~\citep{friedman1993convex}.
We discuss the implications of existing results through our reduction chain, and also propose a simple greedy approach for chasing convex sets with squared 2-norm, leading to the stated path-length bound $\otil(d^{\frac{3}{2}}\sqrt{V_2})$ for linear bandit.

Our main results are summarized in \pref{tab:results},
where the two lower bounds without references are direct implications from known results as discussed in \pref{sec:MAB}. 

\subsection{Related work}
Path-length regret bounds were studied in~\citep{chiang2012online, steinhardt2014adaptivity} for full information problems,
and in~\citep{wei2018more} for MAB and semi-bandit. 
\citet{chiang2013beating} and~\citet{Yang2016tracking} also studies path-length bounds for a partial information setting under the easier two-point bandit feedback.
Our result in \pref{sec:linear_bandit} is the first path-length bound for general linear bandit as far as we know.

Dynamic regret bounds of~\citep{besbes2015non, wei2016tracking} for MAB are also expressed in terms of some path-length measure.
However, the bound is much weaker compared to ours since dynamic regret is a stronger benchmark.
For example, results of~\citep{wei2016tracking} only imply a bound $\order(\sqrt{KTV_\infty})$ for our problem, which is linear in $T$ in the worst case.

Hybrid regularizer was first proposed by~\cite{bubeck2018sparsity} for sparse bandit and bandit with variance bound,
and was also recently used in~\citep{luo2018efficient} for online portfolio.
Our hybrid regularizer is similar to the one of~\citep{bubeck2018sparsity} which is a combination of  Shannon entropy and log-barrier,
but importantly the weight for log-barrier is much larger than that of~\citep{bubeck2018sparsity}.
The purpose of the hybrid regularizer and the role it plays in the analysis are also very different in all these three works.

As mentioned we show a strict distinction between oblivious and adaptive adversary, which is uncommon in online learning.
The other two examples are online learning with switching costs~\citep{cesa2013online}
and best-of-both-worlds results for MAB~\citep{auer2016algorithm}.

\subsection{Problem setup and notation}\label{sec:notation}
The multi-armed bandit problem proceeds for $T$ rounds with $K \leq T$ fixed arms.
In each round $t$, the learner selects one arm $i_t \in [K] \triangleq \{1, 2, \ldots, K\}$,
and simultaneously the adversary decides the loss vector $\ell_t \in [0, 1]^K$ where $\ell_{t,i}$ is the loss for arm $i$ at time $t$.
If $\ell_t$ is selected independently of the learner's previous actions $i_1, \ldots, i_{t-1}$, then the adversary is said to be {\it oblivious};
otherwise the adversary is {\it adaptive}.
In the end of round $t$, the learner suffers and observes the loss of the selected arm $\ell_{t, i_t}$.

The learner's goal is to minimize her (pseudo) regret, 
which is the gap between her total loss and the loss of the best fixed arm,
formally defined as
\[
\reg \triangleq \max_{i^\star \in [K]} 
\E\left[
\sum_{t=1}^T \ell_{t, i_t} - \sum_{t=1}^T \ell_{t, i^\star}
\right]
\]
where the expectation is with respect to the randomness of both the learner and the adversary.

We also consider the more general linear bandit problem,
where the learner's decision set is an arbitrary convex compact set $\Omega \subset \fR^d$.
At each time $t$, the learner picks an action $w_t \in \Omega$
and simultaneously the adversary picks a linear loss function parametrized by $\ell_t \in \calL \subset \fR^d$.
The learner suffers and observes the linear loss $\inner{w_t, \ell_t}$.
Without loss of generality, we assume that $\Omega$ is contained in a unit ball $\calB \triangleq \{z\in \fR^d: \norm{z}\leq 1\}$ for some arbitrary norm $\norm{\cdot}$ and $\calL$ is contained in the dual norm ball $\calB_* \triangleq \{z\in \fR^d: \norm{z}_* \leq 1\}$
(thus the magnitude of the loss for any action is always bounded by $1$).
For linear bandit we consider general adaptive adversary so 
$\ell_t$ can depend on $w_1, \ldots, w_{t-1}$.
The (pseudo) regret is defined in a similar way
\[
\reg \triangleq \max_{w^\star \in \Omega} 
\E\left[
\sum_{t=1}^T \inner{w_t, \ell_t} - \sum_{t=1}^T \inner{w^\star, \ell_t}
\right]
\]
where again the expectation is with respect to the randomness of both the learner and the adversary.

As mentioned we study adaptive regret bounds that depend on the variation
of the loss sequence $\ell_1, \ldots, \ell_T$,
measured by it path-length $V_p = \E\left[\sum_{t=1}^T \norm{\ell_t - \ell_{t-1}}_p\right]$ for some $p \geq 1$,
where we define $\ell_0 = \mathbf{0}$ to be the all-zero vector
and the expectation is taken with respect to the randomness of the adversary
as well as the randomness of the learner in the case of adaptive adversary.
In particular, we consider path-length $V_1$ and $V_\infty$ for MAB and $V_2$ for linear bandit.
Note $V_\infty \leq V_2 \leq V_1$ and also $V_1 \leq K V_\infty$.
For linear bandit, we also consider path-length measured by a general dual norm and denote it by $V_* = \E\left[\sum_{t=1}^T \norm{\ell_t - \ell_{t-1}}_*\right]$.
For simplicity we assume that these quantities are known when tuning the optimal learning rate.

For MAB, we define $\rho_i(t) \triangleq \max\{s \leq t: i_s = i\}$ (or 0 if the set is empty) as the most recent time arm $i$ is selected (prior to round $t+1$).
We use $\basis_i$ to denote the standard basis vector in $K$ dimension with coordinate $i$ being 1 and others being 0.

\section{Path-length Bounds for Multi-armed Bandit}
\label{sec:MAB}

In this section we first show path-length lower bounds for MAB,
followed by our proposed algorithms with new upper bounds.

First note that $\Omega(\sqrt{KV_\infty})$ is a trivial lower bound
for oblivious adversary (and thus also for the more powerful adaptive adversary)
in light of the standard $\sqrt{KT}$ lower bound construction for MAB.
Indeed, for any $\gamma \in [K/T,1]$ and any MAB algorithm,
one can find a loss sequence with $V_\infty = \order(T\gamma)$
and $\reg = \Omega(\sqrt{KT\gamma}) = \Omega(\sqrt{KV_\infty})$,
just by using the standard lower bound construction~\citep{auer2002nonstochastic}
for a game with $T\gamma$ rounds as the first $T\gamma$ rounds,
and setting all losses to be zero for the rest.
Since $KV_\infty \geq V_1$, $\Omega(\sqrt{V_1})$ is clearly also a lower bound.

However, it turns out that for adaptive adversary,
one can prove a stronger lower bound in terms of $V_1$,
as shown in the following theorem.

\begin{theorem}\label{thm:L1_norm_lower_bound}
For any $\gamma \in [K/T,1]$ and any MAB algorithm, there exists an adaptively chosen sequence of $\ell_t$ such that $V_1 = \order(T\gamma)$
and $\reg = \Omega(\sqrt{KT\gamma}) = \Omega(\sqrt{KV_1})$.
\end{theorem}
\begin{proof}
The construction of the loss sequence is as follows.
First uniformly at random pick an arm $i^\star$ as the ``good'' arm. 
Then for each time $t$ and each $i \in [K]$, set
\[
\ell_{t, i} =
\begin{cases}
0, \quad&\text{if $t > T\gamma$} \\
\ell_{t - 1, i}, \quad&\text{else if $t > 1$ and $i \neq i_{t-1}$,} \\
\text{a fresh sample drawn from $\Ber(0.5)$}, \quad&\text{else if $i \neq i^\star$,} \\
\text{a fresh sample drawn from $\Ber\left(0.5 - \frac{1}{4}\sqrt{\frac{K}{T\alpha}}\right)$,} \quad&\text{else.}
\end{cases}
\]
Note that the construction is so that only the first $T\gamma$ rounds matter clearly,
and more importantly for these rounds the regret of the algorithm
is exactly the same as in the case where fresh samples are drawn every time
according to $\Ber\left(0.5 - \frac{1}{4}\sqrt{\frac{K}{T\gamma}}\right)$ for the good arm
and $\Ber(0.5)$ for the others,
because the loss of each arm is a fresh sample from the algorithm's perspective
until this arm is picked and the loss is observed
(in which case a new sample is drawn).
Standard MAB lower bound proofs (see~\citep{auer2002nonstochastic}) show that the regret in this case is $\Omega(\sqrt{KT\gamma})$.
On the other hand, it is also clear that under this construction we have $V_1 = \order(T\gamma)$ since only one coordinate of the loss vector changes for each time $t \leq T\gamma$, which finishes the proof.
\end{proof}

\pref{thm:L1_norm_lower_bound} shows that the algorithm of~\citet[Section 4.1]{wei2018more} is optimal in terms of $V_1$ path-length bound for adaptive adversary.
In the next two subsections we respectively show improvements in terms of $V_\infty$ path-length and oblivious adversary.

\subsection{Improved bounds in terms of $V_\infty$}
\label{sec:max_norm}

We propose a new algorithm that improves the result of~\citet{wei2018more} from $\otil(\sqrt{KV_1})$ to $\otil(\sqrt{KV_\infty})$ for both oblivious and adaptive adversary.
Similar to~\citep{wei2018more}, our algorithm is also based on the optimistic mirror descent framework~\citep{chiang2012online, rakhlin2013online}.
Specifically, optimistic mirror descent for general linear bandit over a decision set $\Omega$ maintains two sequences $x_1, \ldots, x_T$ and $x_1', \ldots, x_T'$
based on the following update rules:
\begin{align*}
x_t &= \argmin_{x\in\Omega} \inner{x, m_t} + D_\psi(x, x_t'), \\
x_{t+1}'  &= \argmin_{x\in\Omega} \inn{x, \hat{\ell}_t} + D_\psi(x, x_t'), 
\end{align*}
where $m_t$ and $\hat{\ell}_t$ are respectively some optimistic prediction and unbiased estimator for the true loss vector $\ell_t$,
$\psi$ is some convex differentiable regularizer and $D_\psi(x, y) = \psi(x) - \psi(y) - \inner{\nabla\psi(y), x - y}$ is the Bregman divergence with respect to $\psi$.

For MAB, $\Omega = \Delta_K$ is the simplex of distributions over $K$ arms
and $\hat{\ell}_t$ is usually set to the unbiased estimator with 
$\hat{\ell}_{t, i} = \frac{\ell_{t,i} - m_{t,i}}{w_{t, i}}\one\{i_t = i\} + m_{t, i}$
where the selected arm $i_t$ is drawn from some final sample distribution $w_t$ (computed based on $x_t$).
\citet{wei2018more} use the log-barrier $\psi(x) = \frac{1}{\eta}\sum_{i=1}^K \ln\frac{1}{x_i}$ with some learning rate $\eta$ as the regularizer and $w_t = x_t$, 
and prove that the regret is bounded by (ignoring constants):
$\frac{K\ln T}{\eta} + \eta\sum_{t=1}^T (\ell_{t,i_t} - m_{t, i_t})^2$.
With $m_{t,i} = \ell_{\rho_i(t-1), i}$ 
(that is, the most recently observed loss for arm $i$),
it is further shown that the regret bound above is at most $\otil(\sqrt{KV_1})$
with the optimal learning rate $\eta$.

Our algorithm makes the following two modifications (see \pref{alg:max_norm} for pseudocode).
First, we simply use the observed loss at time $t$ as the optimistic prediction for {\it all} arms at time $t+1$. 
Formally, we set $m_{t+1,i} = c_{t} \triangleq \ell_{t, i_t}$ for all $i$.
Note that in this case $\inner{x, m_t} = c_{t-1}$ for any $x\in \Delta_K$ and thus $x_t = \argmin_{x\in\Delta_K} \inner{x, m_t} + D_\psi(x, x_t') = x_t'$, meaning that we only need to maintain one sequence (\pref{line:update_of_x}).
Second, instead of using $x_{t+1}$ to sample $i_{t+1}$,
we slightly bias towards the most recently picked arm by moving a small fraction $\alpha_{t+1}$ of each arm's weight to arm $i_{t}$,
where $\alpha_{t+1} = \frac{\alpha(1-c_{t})}{1+\alpha(1-c_{t})}$ for some fixed parameter $\alpha$
(\pref{line:bias}).
Note that the smaller the loss of arm $i_{t}$ is, the more we bias towards this arm,
but the correlation is in some nonlinear form.
Such bias is intuitive in a slowly changing environment where we expect a good arm remains reasonably good for a while. 
In the next theorem we formally prove the improved regret bound of our algorithm.

\begin{algorithm}[t]
\SetAlgoLined
\caption{}
\label{alg:max_norm}
\textbf{Define}: $\psi(x)= \frac{1}{\eta}\sum_{i=1}^K \ln\frac{1}{x_i}$ for some learning rate $\eta$; parameter $\alpha \in (0,1)$.  \\
\textbf{Initialize}: $w_1$ is the uniform distribution, $c_0 = 0$. \\
\For{$t=1, 2, \ldots, T$}{
\nl 	Play $i_t \sim w_t$ and observe $c_{t} = \ell_{t, i_t}$.  \\
\nl 	Construct unbiased estimator $\hat{\ell}_t$ s.t.  $\hat{\ell}_{t,i}= \frac{\ell_{t,i} - c_{t-1}}{w_{t,i}}\one\{i_t=i\} + c_{t-1}$ for all $i$. \\
\nl 	Update $x_{t+1} = \argmin_{x \in \Delta_K} \inn{x, \hat{\ell}_t} + D_\psi(x, x_t)$.  \label{line:update_of_x} \\
\nl 	$w_{t+1} = (1-\alpha_{t+1}) x_{t+1} + \alpha_{t+1} \basis_{i_t}$, where $\alpha_{t+1} = \frac{\alpha(1-c_t)}{1+\alpha(1-c_t)}$. \label{line:bias}
}    
\end{algorithm}

\begin{theorem}\label{thm:max_norm_upper_bound}
\pref{alg:max_norm} with $\eta \leq \frac{1}{162}$ and $\alpha = 8\eta$ 
ensures 
\[
\reg = \order\left(\frac{K\ln T}{\eta} + \eta\E\left[\sum_{t=1}^{T-1}|\ell_{t+1, i_{t}} - \ell_{t, i_{t}}|\right] \right) 
= \order\left(\frac{K\ln T}{\eta} + \eta V_\infty \right) 
\]
for adaptive adversary. 
Picking the optimal $\eta$ leads to regret bound $\order(\sqrt{K V_\infty \ln T} + K\ln T)$.
\end{theorem}

\begin{proof}
We first analyze the regret of the sequence $x_1, \ldots, x_T$ using the analysis of~\citep{wei2018more}.
Specifically by their Theorem~7 and our choice of $m_t$ and $\hat{\ell}_t$,
we have for any arm $i$,
\begin{align}
    \E\left[\sum_{t=1}^T\inn{x_t-\basis_i, \ell_t}\right]    \nonumber 
    &\leq \order\left(\frac{K\ln T}{\eta}\right) + 3\eta \E\left[ \sum_{t=1}^T\sum_{i=1}^K x_{t,i}^2(\hat{\ell}_{t,i}-c_{t-1})^2 \right]   \nonumber \\
    & = \order\left(\frac{K\ln T}{\eta}\right) + 3\eta\E\left[\sum_{t=1}^T \frac{x_{t,i_t}^2}{w_{t,i_t}^2}(c_{t}-c_{t-1})^2 \right] \nonumber \\
    & \leq \order\left(\frac{K\ln T}{\eta}\right) + 4\eta\E\left[\sum_{t=1}^T (c_{t}-c_{t-1})^2\right],  \label{eqn: xt-ei}
\end{align}
where in the last step we use $\frac{x_{t,i}}{w_{t,i}}\leq \frac{1}{1-\alpha_t}\leq 1 + \alpha = 1+ 8\eta \leq \sqrt{4/3}$ by our choice of $\eta$. 
In the rest of the proof we analyze the difference between using $x_t$ and $w_t$.
Specifically we prove
\begin{equation}\label{eqn:bias}
\E\left[\sum_{t=1}^T \inn{w_{t}-x_{t}, \ell_{t}}\right]  
\leq \order(1) + \alpha\E\left[\sum_{t=2}^T |\ell_{t, i_{t-1}} - \ell_{t-1, i_{t-1}}| - \frac{1}{2}\sum_{t=2}^T (c_t - c_{t-1})^2 \right],
\end{equation}
which finishes the proof by combining the two inequalities above
and using $\alpha = 8\eta$.
Indeed, observe that for each time $t > 1$, we have by the definition of $w_t$
\[
\E\left[\inn{w_{t}-x_{t}, {\ell}_{t}}\right] 
= \E\left[\alpha_t\inn{\e_{i_{t-1}}-x_{t},  {\ell}_{t}}\right]
= \E\left[\alpha_t \inn{\e_{i_{t-1}}-w_t, {\ell}_t}\right] 
+ \E\left[\alpha_t \inn{w_t-x_t, {\ell}_t}\right].
\]
Rearranging and plugging the definition of $\alpha_t$ gives
\begin{align*}
    \E\left[\inn{w_{t}-x_{t}, {\ell}_{t}}\right] 
    &=\E\left[\frac{\alpha_t}{1-\alpha_t} \inn{\e_{i_{t-1}}-w_t, {\ell}_t}\right]  \\
    &= \alpha\E\left[(1 - c_{t-1}) \inn{\e_{i_{t-1}}-w_t, {\ell}_t}\right]  \\
    &= \alpha\E\left[(1 - c_{t-1}) (\ell_{t, i_{t-1}} - c_t) \right]  \\
    &= \alpha\E\left[(1 - c_{t-1}) (\ell_{t, i_{t-1}} - c_{t-1} + c_{t-1} - c_t) \right]  \\
    &\leq \alpha\E\left[|\ell_{t, i_{t-1}} - c_{t-1}| + (1-c_{t-1})(c_{t-1}-c_t) \right]  \\
    &= \alpha\E\left[|\ell_{t, i_{t-1}} - c_{t-1}| + (c_{t-1} - c_t - c_{t-1}^2 + c_{t-1}c_t) \right].  
\end{align*}
Summing over $t$, and combining \pref{eqn: xt-ei}, we can bound the regret by (recall $c_{t-1} = \ell_{t-1, i_{t-1}}$)
\begin{align*}
    &\mathcal{O}\left(\frac{K\ln T}{\eta} \right) + 8\eta\E\left[\sum_{t=2}^T |\ell_{t, i_{t-1}}-\ell_{t-1, i_{t-1}}|  \right] 
            + \E\left[4\eta \sum_{t=2}^T (c_t-c_{t-1})^2 + 8\eta \sum_{t=2}^T (-c_{t-1}^2+c_{t-1}c_t)\right] \\
    &=\mathcal{O}\left(\frac{K\ln T}{\eta} \right) + 8\eta\E\left[\sum_{t=2}^T |\ell_{t, i_{t-1}}-\ell_{t-1, i_{t-1}}|  \right], \tag{telescoping} 
\end{align*}
which finishes the proof.
\end{proof}

\subsection{Improved bounds in terms of $V_1$ for oblivious adversary}
\label{sec:L1_norm}

Next we come back to $V_1$ path-length bound and show that despite the lower bound of \pref{thm:L1_norm_lower_bound} for adaptive adversary,
one can still improve the regret for oblivious adversary when the path-length is large.

Our algorithm (see \pref{alg:L1_norm} for the pseudocode) still follows the general optimistic mirror descent framework (\pref{line:OMD1} and \pref{line:OMD2}).
The novelty is that we divide all arms into two groups: the minority group $\calS_t$ consisting of arms with weight $x_{t,i}$ smaller than some parameter $\beta$,
and the majority group $[K]\backslash\calS_t$.\footnote{%
The concept of majority and minority groups is reminiscent of the recent work~\citep{allen2018make} on first-order regret bounds for contextual bandits.
}
At a high level our algorithm uses the same strategy as \pref{alg:max_norm} for the minority group and a different strategy for the majority group,
discussed in detail below.

\paragraph{Optimistic prediction (\pref{line:optimistic_prediction}).}
For arms in the minority group $\calS_t$, we use $m_{t+1, i} = c_{\tau(t)}$ as the optimistic prediction for time $t+1$, 
where $\tau(t) \triangleq \max\{\tau \leq t: i_\tau \in \calS_{\tau-1}\}$ is basically the most recent time we selected a minority arm,
and $c_t \triangleq\ell_{t, i_t}$ is the loss of the algorithm at time $t$ (same as \pref{alg:max_norm}).
On the other hand, for arms in the majority group $[K]\backslash \calS_t$,
just like \citep{wei2018more} we use their most recently observed loss as the optimistic prediction, that is, $m_{t+1, i} = \ell_{\rho_i(t), i}$.
This is very natural since intuitively majority arms are selected more often by definition
and therefore their last observed loss could be a good proxy for the current loss.

\paragraph{Slight bias (\pref{line:bias2}).}
Among the minority, we also bias towards the most recently picked one $i_{\tau(t)}$ by moving a fraction $\alpha_{t+1}$ of the weights of all arms in $\calS_t \backslash \{i_{\tau(t)}\}$ to arm $i_{\tau(t)}$, just in the same way as \pref{alg:max_norm}.\footnote{It is possible that $i_{\tau(t)}$ is not in the current minority group $\calS_t$ though.}
For the rest of the arms we simply set $w_{t+1, i} = x_{t+1, i}$.

\paragraph{Regularizer.}
Our algorithm uses a hybrid regularizer $\psi(x)= \frac{1}{\eta}\sum_{i=1}^K \ln\frac{1}{x_i} + \frac{K}{\eta}\sum_{i=1}^K x_i\ln x_i$. 
Roughly speaking, in our analysis we apply the log-barrier $\frac{1}{\eta}\sum_{i=1}^K \ln\frac{1}{x_i}$ to the minority group
and the (negative) Shannon entropy $\frac{K}{\eta}\sum_{i=1}^K x_i\ln x_i$ to the majority group (see \pref{lem:OMD}). 
This regularizer is similar to the one used by~\citet{bubeck2018sparsity}.
The difference is that in~\citep{bubeck2018sparsity},
the purpose of the hybrid regularizer is to ensure that the algorithm is stable 
(in the sense of \pref{lem:stability}),
and for that purpose it suffices to set the coefficient of the log-barrier to be as small as $K$.
On the other hand, 
we use a much larger coefficient $\frac{1}{\eta}$ for the log-barrier.
In fact, this is in some sense the most natural choice since it leads to the smallest variance term for mirror descent while keeping the same regularization overhead as the entropy part (see \pref{lem:OMD} for details).
Our analysis exactly exploits the smaller variance for the minority group due to this hybrid regularizer. 

\begin{algorithm}[t]
\DontPrintSemicolon
\setcounter{AlgoLine}{0}
\caption{}
\label{alg:L1_norm}
\textbf{Define}: $\psi(x)= \frac{1}{\eta}\sum_{i=1}^K \ln\frac{1}{x_i} + \frac{K}{\eta}\sum_{i=1}^K x_i\ln x_i $ for some learning rate $\eta$; parameters $\alpha, \beta \in (0,1)$. \\
\textbf{Initialize}: $x_1'$, $x_1$, $w_1$ are uniform distributions, $m_1 = \mathbf{0}$,
and $\calS_0=[K]$. \\ 
\For{$t=1, 2, \ldots, T$}{
\nl	Play $i_t \sim w_t$ and observe $c_t = \ell_{t, i_t}$. \\
\nl    Let $\tau(t) = \max\{\tau \leq t: i_\tau \in \calS_{\tau-1} \}$
       and $\calS_{t} = \{i\in [K]: x_{t,i} < \beta\}$. \\
        \ 
\nl	Construct unbiased estimator $\hat{\ell}_t$ s.t.  $\hat{\ell}_{t,i}= \frac{\ell_{t,i} - m_{t,i}}{w_{t,i}}\one\{i_t=i\} + m_{t,i}$ for all $i$. \\
\nl	Construct optimistic prediction
	$m_{t+1,i}=\begin{cases}
	c_{\tau(t)} \quad&\text{if $i \in \calS_t$,} \\
	\ell_{\rho_i(t), i} \quad&\text{else.}
	\end{cases}
	$ \label{line:optimistic_prediction} \\
\nl	Update $x_{t+1}' = \argmin_{x \in \Delta_K} \inn{x, \hat{\ell}_t} + D_\psi(x, x_t')$. \label{line:OMD1} \\
\nl     Update $x_{t+1} = \argmin_{x \in \Delta_K} \inner{x, m_{t+1}} + D_\psi(x, x_{t+1}')$.  \label{line:OMD2} \\
\nl     Let $\alpha_{t+1} = \frac{\alpha(1-c_{\tau(t)})}{1+\alpha(1-c_{\tau(t)})}$ and $w_{t+1}$ be such that \\
     \[
     w_{t+1, i} = \begin{cases}
     x_{t+1,i}(1-\alpha_{t+1}) \quad&\text{if $i \in \calS_{t} \backslash \{i_{\tau(t)}\}$,}\\
     x_{t+1, i} + \alpha_{t+1}\sum_{j\in \calS_{t} \backslash \{i_{\tau}\}}x_{t+1,j} \quad&\text{else if $i = i_{\tau(t)}$,}\\
     x_{t+1, i} \quad&\text{else.} \\
     \end{cases}
     \] \label{line:bias2}
}    
\end{algorithm}

The next theorem shows the improved regret of \pref{alg:L1_norm}
(see \pref{app:proofs} for the proof).

\begin{theorem}\label{thm:1_norm_upper_bound}
\pref{alg:L1_norm} with $\eta \leq \min\left\{\frac{1}{K}, \frac{1}{162}\right\}$ and $\alpha = 8\eta$ ensures
\[
\reg = \order\left(\frac{K\ln T}{\eta} + \frac{\eta V_1}{K\beta} + \eta\sqrt{\beta TV_1} \right)
\]
for oblivious adversary.
Picking the optimal parameters leads to regret $\otil\left(\LOneBound + K^2\right)$.
\end{theorem}

The new regret bound is smaller than $\otil(\sqrt{KV_1})$~\citep{wei2018more} as long as $V_1 \geq T/K$.
This serves as a proof of concept that improved regret in terms of $V_1$ is possible for oblivious adversary,
and we expect that even $\otil(\sqrt{V_1})$ is achievable using our techniques,
although we do not have a simple algorithm achieving it yet.

\section{Path-length Bounds for Linear Bandit}
\label{sec:linear_bandit}

In this section we move on to the more general linear bandit problem.
Recall that the decision sets of the learner and the adversary are assumed to be contained in a unit norm ball $\calB$ and the dual norm ball $\calB_*$ respectively.
Our algorithm (see \pref{alg:linear_bandit} for the pseudocode) is based on the optimistic SCRiBLe algorithm~\citep{abernethy2008competing, hazan2011better, rakhlin2013online} with new optimistic predictions.

Specifically, optimistic SCRiBLe is again an instance of the general optimistic mirror descent reviewed in \pref{sec:max_norm}.
The regularizer is any $\nu$-self-concordant barrier of the decision set $\Omega$ for some $\nu > 0$.
Having the point $x_t$, the algorithm uniformly at random selects one of the $2d$ endpoints of the principal axes of the unit Dikin ellipsoid centered at $x_t$,
as the final action $w_t$ (Lines~\ref{line:sampling1},~\ref{line:sampling2} and~\ref{line:sampling3}).
After observing the loss $c_t = \inn{w_t, \ell_t}$,
the algorithm then constructs an unbiased loss estimator (\pref{line:estimator})
and uses it in the next optimistic mirror descent update (\pref{line:OMD_eta_1} and \pref{line:OMD_eta_2}, note that the learning rate $\eta$ is explicitly spelled out here).
We refer the readers to~\citep{abernethy2008competing, rakhlin2013online} for more detailed explanation of the (optimistic) SCRiBLe algorithm.

\begin{algorithm}[t]
\setcounter{AlgoLine}{0}
\caption{}
\label{alg:linear_bandit}
\textbf{Define}: $\psi(x)$ is a $\nu$-self-concordant barrier; learning rate $\eta$. \\
\textbf{Initialize}: $x_1 = x_1' = \argmin_{x\in\Omega}\psi(x)$ and $m_1 = \mathbf{0}$. \\ 
\For{$t=1, 2, \ldots, T$}{
\nl     Compute eigendecomposition $\nabla^2\psi(x_t) = \sum_{i=1}^d \lambda_{t,i}v_{t,i}v_{t,i}^\top$. \label{line:sampling1} \\
\nl		Sample $i_t \in [d]$ and $\sigma_t \in \{-1, +1\}$ uniformly at random. \label{line:sampling2}\\
\nl 	Play $w_t = x_t + \frac{\sigma_t}{\sqrt{\lambda_{t,i_t}}}v_{t,i_t}$ and observe $c_t = \inn{w_t, \ell_t}$.  \label{line:sampling3}\\
\nl 	Construct unbiased estimator $\hat{\ell}_t = d(c_t - \inner{w_t, m_t})\sigma_t\sqrt{\lambda_{t,i_t}}v_{t,i_t} + m_t$. \label{line:estimator}\\
\nl		Set $m_{t+1} 
\begin{cases} 
=\proj_{\calB}\left(m_t - \frac{1}{4}(\inn{w_t, m_t} - c_t)w_t\right) &\text{ (Option I)} \\
=\proj_{\calK_{t+1}}(m_t)$ where $\calK_{t+1} \triangleq \left\{m \in \calB: \inner{w_t, m} = c_t\right\}  &\text{ (Option II)} \\
\text{via the convex body chasing algorithm of~\citep{sellke2019chasing}} &\text{ (Option III)}
\end{cases} 
$\label{line:chasing} \\
\nl 	Update $x_{t+1}'  = \argmin_{x\in\Omega} \eta\inn{x, \hat{\ell}_t} + D_\psi(x, x_t')$.  \label{line:OMD_eta_1} \\
\nl 	Update $x_{t+1} = \argmin_{x\in\Omega} \eta\inn{x, m_{t+1}} + D_\psi(x, x_{t+1}')$.  \label{line:OMD_eta_2}
}    
\end{algorithm}

For any optimistic prediction sequence $m_1, \ldots, m_T$,
~\citet{rakhlin2013online} shows that the regret of optimistic SCRiBLe is bounded as
\begin{equation}\label{eqn:SCRiBLe}
\reg = \order\left(\frac{\nu\ln T}{\eta} + \eta d^2 \E\left[ \sum_{t=1}^T \inn{w_t, \ell_t - m_t}^2 \right]\right).
\end{equation}
It remains to specify how to pick the optimistic predictions $m_1, \ldots, m_T$ such that the last term above $\sum_{t=1}^T \inn{w_t, \ell_t - m_t}^2$ is close to the path-length of the loss sequence $\ell_1, \ldots, \ell_T$.
This is trivial in the full information setting where one observes $\ell_t$ at the end of round $t$ and can simply set $m_t = \ell_{t-1}$.
In the bandit setting, however, only $c_t = \inn{w_t, \ell_t}$ is observed and the problem becomes more challenging.
In the next subsections we propose two approaches,
one through a reduction to obtaining dynamic regret in an online learning problem with full information, 
and another via a further reduction to an instance of convex body chasing.
As a side result, we obtain new dynamic regret bounds that may be of independent interest.

\subsection{Reduction to dynamic regret}\label{sec:dynamic_regret}

\citet{rakhlin2013online} suggest treating the problem of selecting $m_t$ as another online learning problem.
Specifically, consider the following online learning formulation:
at each time $t$ the algorithm selects $m_t \in \calB_*$ and then observes the loss function $f_t(m) = \inn{w_t, \ell_t - m}^2 = (c_t - \inn{w_t, m})^2$.
Note that this is a full information problem even though $\ell_t$ is unknown
and is in fact the standard problem of online linear regression with squared loss.
Further observe that applying Online Newton Step~\citep{hazan2007logarithmic} to learn $m_t$ ensures
\begin{align*}
&\sum_{t=1}^T f_t(m_t)
\leq \min_{m^\star\in \calB_*} \sum_{t=1}^T f_t(m^\star) + \order(d\ln T) 
\leq \min_{m^\star\in \calB_*} \sum_{t=1}^T \norm{\ell_t - m^\star}_*^2 + \order(d\ln T)
\end{align*}
Picking $m^\star= \frac{1}{T}\sum_{s=1}^T\ell_s$ and combining the above with \pref{eqn:SCRiBLe} immediately recover the main result of~\citep{hazan2011better} with a different approach.
(This observation was not made in~\citep{rakhlin2013online} though.)

However, competing with a fixed $m^\star$ is not adequate for getting path-length bound.
Instead in this case we need a {\it dynamic regret} bound~\citep{zinkevich2003online} that allows the algorithm to compete with some sequence $m_1^\star, \ldots, m_T^\star$ instead of a fixed $m^\star$. 
Typical dynamic regret bounds depend on either the variation of the loss functions or the competitor sequence~\citep{jadbabaie2015online, mokhtari2016online, Yang2016tracking, zhang2017improved, zhang2018adaptive},
and here we need the latter one.
Specifically, when $\calB_*$ is the unit 2-norm ball, \citet{Yang2016tracking} discover that projected gradient descent with a constant learning rate ensures for any minimizer sequence $m_1^\star \in \argmin_{m\in\calB_*}f_1(m), \ldots, m_T^\star\in \argmin_{m\in\calB_*}f_T(m)$,
\begin{equation}\label{eqn:dynamic_regret}
\sum_{t=1}^T f_t(m_t) -
\sum_{t=1}^T f_t(m_t^\star) \leq  \order\left(L\sum_{t=2}^T \norm{m_t^\star-m_{t-1}^\star}_2\right)
\end{equation}
as long as the following assumption holds:
\begin{assumption}\label{ass:smooth}
Each $f_t$ is convex and $L$-smooth (that is, for any $m, m' \in \calB_*$, 
$f_t(m) \leq f_t(m') + \inn{\nabla f_t(m'), m-m'} + \frac{L}{2}\norm{m-m'}_2^2$).
Additionally, $\nabla f_t(m^\star) = 0$ for any $m^\star \in \argmin_{m\in\calB_*}f_t(m)$.
\end{assumption}

It is clear that $f_t(m) = (c_t - \inn{w_t, m})^2$ satisfies \pref{ass:smooth} with $L=4$.
Also note that Option I in \pref{line:chasing} is exactly doing projected gradient descent with $f_t$ (we define $\proj_\calK(m) = \argmin_{m' \in \calK}\norm{m - m'}$).
Therefore picking $m_t^\star = \ell_t$ and combining \pref{eqn:dynamic_regret} and \pref{eqn:SCRiBLe} immediately imply the following.

\begin{corollary}\label{cor:linear_bandit}
When $\calB$ and $\calB_*$ are unit 2-norm balls, \pref{alg:linear_bandit} with Option I ensures
$
\reg = \order\left(\frac{\nu\ln T}{\eta} + \eta d^2 \E\left[\sum_{t=1}^T \norm{\ell_{t} - \ell_{t-1}}_2\right]\right),
$
which is of order $\otil\left( d\sqrt{\nu V_2} \right)$ with the optimal $\eta$.
\end{corollary}

To deal with the case when $\calB$ and $\calB_*$ are arbitrary primal-dual norm balls, we require dynamic regret bounds that are similar to \pref{eqn:dynamic_regret} but hold for an arbitrary norm.
We are not aware of any such existing results. Instead, in the next section we provide a solution via a further reduction to convex body chasing.

\subsection{Further reduction to convex body chasing}\label{sec:chasing}

Next we provide an alternative approach to obtain dynamic regret similar to \pref{eqn:dynamic_regret}, via a reduction to convex body chasing~\citep{friedman1993convex}, which in turn leads to a different approach for obtaining path-length bound for linear bandit.

We first describe the general convex body chasing problem
(overloading some of our notations for convenience).
For each time $t=1, \ldots, T$,
the algorithm is presented with a convex set $\calK_t$ in some metric space and then needs to select a point $m_t \in \calK_t$.
The algorithm performance is measured by the total movement cost 
$\sum_{t=1}^{T-1} \dist(m_t, m_{t+1})$ where $\dist(\cdot, \cdot)$ is some distance function (usually a metric).
The algorithm is said to be $\omega$-competitive if its total movement cost is at most $\omega\sum_{t=1}^{T-1} \dist(m_t^\star, m_{t+1}^\star)$ for any sequence $m_1^\star \in \calK_1, \ldots, m_T^\star \in \calK_T$.

Now for a sequence of convex functions $f_1, \ldots, f_T$ defined on $\calB_*$ that are $G$-Lipschitz with respect to norm $\norm{\cdot}_*$,
suppose we have a $\omega$-competitive algorithm for $\calK_t = \argmin_{m\in\calB_*} f_{t-1}(m)$ ($\calK_1=\calB_*$)  and $\dist(m, m') = \norm{m-m'}_*$, to produce a sequence $m_1, \ldots, m_T$, then it holds for any minimizer sequence $m_1^\star \in \calK_2 , \ldots, m_T^\star\in \calK_{T+1} $, 
\begin{align*}
\sum_{t=1}^{T-1} f_t(m_t) -   f_t(m_t^\star) 
= \sum_{t=1}^{T-1} f_t(m_t) -   f_t(m_{t+1})  
\leq G \sum_{t=1}^{T-1} \norm{m_t - m_{t+1}}_* 
\leq G\omega \sum_{t=1}^{T-1} \norm{m_t^\star-m_{t+1}^\star}_*, 
\end{align*}
where the first step is by the fact $m_t^\star, m_{t+1} \in \calK_{t+1}$,
the second step is due to $G$-Lipschitzness, and the last step is by $\omega$-competitiveness.
This is exactly a dynamic regret bound of the form \pref{eqn:dynamic_regret},
thus showing a reduction from dynamic regret to convex body chasing,
under {\it only the Lipschitzness assumption} on $f_t$.
%
%
%
Furthermore, recent work of~\citep{sellke2019chasing} shows that for any norm $\norm{\cdot}_*$, there exists an algorithm with competitive ratio $\omega=d$.
This immediately implies the following new dynamic regret result.
\begin{proposition}
For an online convex optimization problem with convex loss functions
$f_1, \ldots, f_T$ that are defined over a subset of $\calB_*$ and are $G$-Lipschitz with respect to some norm $\norm{\cdot}_*$,
there exists an online learning algorithm that selects $m_1, \ldots, m_T$ such that
\begin{equation*}
\sum_{t=1}^T f_t(m_t) -
\sum_{t=1}^T f_t(m_t^\star) \leq  \order\left(Gd \sum_{t=1}^{T-1} \norm{m_t^\star-m_{t+1}^\star}_*\right)
\end{equation*}
for any minimizer sequence $m_1^\star \in \argmin_{m}f_1(m), \ldots, m_T^\star\in \argmin_{m}f_T(m)$.\footnote{%
According to~\citep{sellke2019chasing}, when $\calB_*$ is the 2-norm ball, the dependence on $d$ can be improved to $\sqrt{d\ln T}$.
}
\end{proposition}
We note that this is the first dynamic regret bound in terms of the variation $\sum_{t=1}^{T-1} \norm{m_t^\star-m_{t+1}^\star}_*$ for an arbitrary norm, without any explicit dependence on $T$, and under only the Lipschitzness assumption.
Combining this result with \pref{eqn:SCRiBLe} and picking $m_t^\star = \ell_t$ immediately imply the following path-length regret bound.
\begin{corollary}\label{cor:linear_bandit3}
\pref{alg:linear_bandit} with Option III 
ensures
$
\reg = \order\left(\frac{\nu\ln T}{\eta} + \eta d^3 \E\left[\sum_{t=1}^T \norm{\ell_{t} - \ell_{t-1}}_*\right]\right),
$
which is of order $\otil\left( d^{3/2}\sqrt{\nu V_*} \right)$ with the optimal $\eta$.
\end{corollary}

While the result above holds for an arbitrary norm, it is $\sqrt{d}$ worse than that of \pref{cor:linear_bandit} when the norm is 2-norm.
Below we make another observation that when $\calB_*$ is the 2-norm ball and when \pref{ass:smooth} holds,
the problem in fact reduces to a slightly different convex body chasing problem that admits a constant competitive ratio in some sense.
In particular, note that if $m_t^\star, m_{t+1} \in \calK_{t+1}$ for all $t$,
then by smoothness and $\nabla f_t(m_{t+1})=0$ it holds
\begin{align}
&\sum_{t=1}^{T-1} f_t(m_t) -   f_t(m_t^\star) 
= \sum_{t=1}^{T-1} f_t(m_t) -   f_t(m_{t+1}) \notag \\
&\leq \sum_{t=1}^{T-1} \inn{\nabla f_t(m_{t+1}), m_t - m_{t+1}} + \frac{L}{2}\norm{m_t - m_{t+1}}_2^2 
= \frac{L}{2}\sum_{t=1}^{T-1} \norm{m_t - m_{t+1}}_2^2. \label{eqn:intermediate}
\end{align}
Therefore, if we had a chasing algorithm with $\calK_t$'s as the sets and {\it squared} 2-norm as the distance function,
we would have a dynamic regret in terms of $\sum_{t=1}^{T-1} \norm{m_t^\star - m_{t+1}^\star}_2^2$.
It turns out that this is not possible in general.
However, we propose a very natural greedy approach that is ``competitive'' in a slightly weaker sense where we measure the movement cost of the algorithm by squared 2-norm and the movement cost of the benchmark by 2-norm.
More concretely we prove the following (see \pref{app:chasing} for the proof):

\begin{theorem}[Convex body chasing with squared 2-norm]\label{thm:chasing}
Suppose $\calB_*$ is the unit 2-norm ball. Let $\calK_1, \ldots, \calK_T \subset \calB_*$ be a sequence of convex sets and $m_{t} = \proj_{\calK_{t}}(m_{t-1})$ (with $m_{0} \in \calB_*$ being arbitrary).
Then the following competitive guarantee holds for any sequence $m_1^\star \in \calK_1, \ldots, m_T^\star \in \calK_T$:
$
\sum_{t=1}^{T-1} \norm{m_t - m_{t+1}}_2^2
\leq 4 + 6\sum_{t=1}^{T-1} \norm{m_t^\star - m_{t+1}^\star}_2.
$
\end{theorem}

Combining \pref{eqn:intermediate} and \pref{thm:chasing} then recovers the dynamic regret bound of \pref{eqn:dynamic_regret} with a different approach compared to~\citep{Yang2016tracking} (under the same \pref{ass:smooth}).
Note that Option II of \pref{line:chasing} exactly implements the greedy projection approach of \pref{thm:chasing}.
Therefore according to the discussions in \pref{sec:dynamic_regret} we have:

\begin{corollary}\label{cor:linear_bandit2}
When $\calB$ and $\calB_*$ are unit 2-norm balls, \pref{alg:linear_bandit} with Option II 
ensures
$
\reg = \order\left(\frac{\nu\ln T}{\eta} + \eta d^2 \E\left[\sum_{t=1}^T \norm{\ell_{t} - \ell_{t-1}}_2\right]\right),
$
which is of order $\otil\left( d\sqrt{\nu V_2} \right)$ with the optimal $\eta$.
\end{corollary}

It is well known that any convex body in $d$ dimension admits an $\order(d)$-self-concordant barrier,
and therefore \pref{alg:linear_bandit} admits a regret bound $\otil\left( d^{3/2}\sqrt{V_2} \right)$ or more generally $\otil\left(d^2\sqrt{V_*}\right)$.



\acks{The authors would like to thank all the anonymous reviewers for their valuable comments. HL and CYW are supported by NSF Grant \#1755781. }

\bibliography{ref}

\begin{thebibliography}{27}
\providecommand{\natexlab}[1]{#1}
\providecommand{\url}[1]{\texttt{#1}}
\expandafter\ifx\csname urlstyle\endcsname\relax
  \providecommand{\doi}[1]{doi: #1}\else
  \providecommand{\doi}{doi: \begingroup \urlstyle{rm}\Url}\fi

\bibitem[Abernethy et~al.(2008)Abernethy, Hazan, and
  Rakhlin]{abernethy2008competing}
Jacob~D Abernethy, Elad Hazan, and Alexander Rakhlin.
\newblock Competing in the dark: An efficient algorithm for bandit linear
  optimization.
\newblock In \emph{Conference on Learning Theory}, pages 263--274, 2008.

\bibitem[Allen-Zhu et~al.(2018)Allen-Zhu, Bubeck, and Li]{allen2018make}
Zeyuan Allen-Zhu, S{\'e}bastien Bubeck, and Yuanzhi Li.
\newblock Make the minority great again: First-order regret bound for
  contextual bandits.
\newblock In \emph{International Conference on Machine Learning}, 2018.

\bibitem[Auer and Chiang(2016)]{auer2016algorithm}
Peter Auer and Chao-Kai Chiang.
\newblock An algorithm with nearly optimal pseudo-regret for both stochastic
  and adversarial bandits.
\newblock In \emph{Conference on Learning Theory}, pages 116--120, 2016.

\bibitem[Auer et~al.(2002)Auer, Cesa-Bianchi, Freund, and
  Schapire]{auer2002nonstochastic}
Peter Auer, Nicolo Cesa-Bianchi, Yoav Freund, and Robert~E Schapire.
\newblock The nonstochastic multiarmed bandit problem.
\newblock \emph{SIAM journal on computing}, 32\penalty0 (1):\penalty0 48--77,
  2002.

\bibitem[Besbes et~al.(2015)Besbes, Gur, and Zeevi]{besbes2015non}
Omar Besbes, Yonatan Gur, and Assaf Zeevi.
\newblock Non-stationary stochastic optimization.
\newblock \emph{Operations research}, 63\penalty0 (5):\penalty0 1227--1244,
  2015.

\bibitem[Bubeck et~al.(2012)Bubeck, Cesa-Bianchi, and
  Kakade]{bubeck2012towards}
S{\'e}bastien Bubeck, Nicolo Cesa-Bianchi, and Sham Kakade.
\newblock Towards minimax policies for online linear optimization with bandit
  feedback.
\newblock In \emph{Conference on Learning Theory}, 2012.

\bibitem[Bubeck et~al.(2018)Bubeck, Cohen, and Li]{bubeck2018sparsity}
S{\'{e}}bastien Bubeck, Michael~B. Cohen, and Yuanzhi Li.
\newblock Sparsity, variance and curvature in multi-armed bandits.
\newblock In \emph{International Conference on Algorithmic Learning Theory},
  2018.

\bibitem[Cesa-Bianchi et~al.(2013)Cesa-Bianchi, Dekel, and
  Shamir]{cesa2013online}
Nicolo Cesa-Bianchi, Ofer Dekel, and Ohad Shamir.
\newblock Online learning with switching costs and other adaptive adversaries.
\newblock In \emph{Advances in Neural Information Processing Systems}, pages
  1160--1168, 2013.

\bibitem[Chiang et~al.(2012)Chiang, Yang, Lee, Mahdavi, Lu, Jin, and
  Zhu]{chiang2012online}
Chao-Kai Chiang, Tianbao Yang, Chia-Jung Lee, Mehrdad Mahdavi, Chi-Jen Lu, Rong
  Jin, and Shenghuo Zhu.
\newblock Online optimization with gradual variations.
\newblock In \emph{Conference on Learning Theory}, 2012.

\bibitem[Chiang et~al.(2013)Chiang, Lee, and Lu]{chiang2013beating}
Chao-Kai Chiang, Chia-Jung Lee, and Chi-Jen Lu.
\newblock Beating bandits in gradually evolving worlds.
\newblock In \emph{Conference on Learning Theory}, pages 210--227, 2013.

\bibitem[Cutkosky and Orabona(2018)]{cutkosky2018black}
Ashok Cutkosky and Francesco Orabona.
\newblock Black-box reductions for parameter-free online learning in banach
  spaces.
\newblock In \emph{Conference on Learning Theory}, 2018.

\bibitem[Dani et~al.(2008)Dani, Kakade, and Hayes]{dani2008price}
Varsha Dani, Sham~M Kakade, and Thomas~P Hayes.
\newblock The price of bandit information for online optimization.
\newblock In \emph{Advances in Neural Information Processing Systems}, pages
  345--352, 2008.

\bibitem[Friedman and Linial(1993)]{friedman1993convex}
Joel Friedman and Nathan Linial.
\newblock On convex body chasing.
\newblock \emph{Discrete \& Computational Geometry}, 9\penalty0 (3):\penalty0
  293--321, 1993.

\bibitem[Hazan and Kale(2011)]{hazan2011better}
Elad Hazan and Satyen Kale.
\newblock Better algorithms for benign bandits.
\newblock \emph{Journal of Machine Learning Research}, 12\penalty0
  (Apr):\penalty0 1287--1311, 2011.

\bibitem[Hazan et~al.(2007)Hazan, Agarwal, and Kale]{hazan2007logarithmic}
Elad Hazan, Amit Agarwal, and Satyen Kale.
\newblock Logarithmic regret algorithms for online convex optimization.
\newblock \emph{Machine Learning}, 69\penalty0 (2-3):\penalty0 169--192, 2007.

\bibitem[Jadbabaie et~al.(2015)Jadbabaie, Rakhlin, Shahrampour, and
  Sridharan]{jadbabaie2015online}
Ali Jadbabaie, Alexander Rakhlin, Shahin Shahrampour, and Karthik Sridharan.
\newblock Online optimization: Competing with dynamic comparators.
\newblock In \emph{Artificial Intelligence and Statistics}, pages 398--406,
  2015.

\bibitem[Luo et~al.(2018)Luo, Wei, and Zheng]{luo2018efficient}
Haipeng Luo, Chen-Yu Wei, and Kai Zheng.
\newblock Efficient online portfolio with logarithmic regret.
\newblock In \emph{Advances in Neural Information Processing Systems}, 2018.

\bibitem[Mokhtari et~al.(2016)Mokhtari, Shahrampour, Jadbabaie, and
  Ribeiro]{mokhtari2016online}
Aryan Mokhtari, Shahin Shahrampour, Ali Jadbabaie, and Alejandro Ribeiro.
\newblock Online optimization in dynamic environments: Improved regret rates
  for strongly convex problems.
\newblock In \emph{55th {IEEE} Conference on Decision and Control}, pages
  7195--7201, 2016.

\bibitem[Rakhlin and Sridharan(2013)]{rakhlin2013online}
Alexander Rakhlin and Karthik Sridharan.
\newblock Online learning with predictable sequences.
\newblock In \emph{Conference on Learning Theory}, pages 993--1019, 2013.

\bibitem[Sellke(2019)]{sellke2019chasing}
Mark Sellke.
\newblock Chasing convex bodies optimally.
\newblock \emph{arXiv preprint arXiv:1905.11968}, 2019.

\bibitem[Steinhardt and Liang(2014)]{steinhardt2014adaptivity}
Jacob Steinhardt and Percy Liang.
\newblock Adaptivity and optimism: An improved exponentiated gradient
  algorithm.
\newblock In \emph{International Conference on Machine Learning}, pages
  1593--1601, 2014.

\bibitem[Wei and Luo(2018)]{wei2018more}
Chen-Yu Wei and Haipeng Luo.
\newblock More adaptive algorithms for adversarial bandits.
\newblock In \emph{Conference on Learning Theory}, 2018.

\bibitem[Wei et~al.(2016)Wei, Hong, and Lu]{wei2016tracking}
Chen-Yu Wei, Yi-Te Hong, and Chi-Jen Lu.
\newblock Tracking the best expert in non-stationary stochastic environments.
\newblock In \emph{Advances in Neural Information Processing Systems}, pages
  3972--3980, 2016.

\bibitem[Yang et~al.(2016)Yang, Zhang, Jin, and Yi]{Yang2016tracking}
Tianbao Yang, Lijun Zhang, Rong Jin, and Jinfeng Yi.
\newblock Tracking slowly moving clairvoyant: Optimal dynamic regret of online
  learning with true and noisy gradient.
\newblock In \emph{International Conference on Machine Learning}, pages
  449--457, 2016.

\bibitem[Zhang et~al.(2017)Zhang, Yang, Yi, Rong, and Zhou]{zhang2017improved}
Lijun Zhang, Tianbao Yang, Jinfeng Yi, Jing Rong, and Zhi-Hua Zhou.
\newblock Improved dynamic regret for non-degenerate functions.
\newblock In \emph{Advances in Neural Information Processing Systems}, pages
  732--741, 2017.

\bibitem[Zhang et~al.(2018)Zhang, Lu, and Zhou]{zhang2018adaptive}
Lijun Zhang, Shiyin Lu, and Zhi-Hua Zhou.
\newblock Adaptive online learning in dynamic environments.
\newblock In \emph{Advances in Neural Information Processing Systems}, pages
  1330--1340, 2018.

\bibitem[Zinkevich(2003)]{zinkevich2003online}
Martin Zinkevich.
\newblock Online convex programming and generalized infinitesimal gradient
  ascent.
\newblock In \emph{International Conference on Machine Learning}, 2003.

\end{thebibliography}

\appendix

\section{Open Problems}
In this work we provide several improvements on path-length bounds for bandit.
There are several open problems in this direction and we hope that the techniques we develop here are useful for solving these problems.
First, our upper bound for MAB with oblivious adversary has some dependence on $T$.
Whether one can remove this dependence,
and in particular, whether $\order(\sqrt{V_1})$ is achievable are clear open problems.

Second, all existing path-length bounds for bandit are ``first-order'',
while smaller bounds in terms of ``second-order'' path-length $\E\left[\sum_{t=1}^T \norm{\ell_t - \ell_{t-1}}_p^2\right]$ for some $p \geq 1$
are only known for full information problems.
It is therefore natural to ask whether second-order path-length bounds are achievable for bandits, or there is a distinction here due to the partial information feedback.

Third, in light of the bound $\sqrt{d\sum_{t=1}^T \inner{w^\star, \ell_t}^2}$ of~\citep{cutkosky2018black} for full information setting with linear losses (where $w^\star$ is the competitor the regret is with respect to),
it is also very natural to ask if path-length bounds of the form $\text{poly}(d)\sqrt{\sum_{t=1}^T |\inner{w^\star, \ell_t - \ell_{t-1}}|}$ are possible (for either full-information or bandit feedback).

\section{Proof of \pref{thm:1_norm_upper_bound}}
\label{app:proofs}

We outline the proof below and defer several technical lemmas to the next subsection. \\

\begin{proof}[of \pref{thm:1_norm_upper_bound}]
Define $B_t = \one\{i_t \in \calS_{t-1}\}$.
Standard optimistic mirror descent analysis (see \pref{lem:OMD}) shows that the hybrid regularizer ensures the following regret bound for the $x_t$ sequence: for any arm $i$,
\begin{equation}\label{eqn:OMD_bound}
\E\left[\sum_{t=1}^T \inn{x_t-\basis_i, {\ell}_t}\right]
\leq \mathcal{O}\left(\frac{K\ln T}{\eta}\right) + 
4\eta\E\left[\sum_{t:B_t=0} \frac{(\ell_{t,i_t}-m_{t,i_t})^2}{Kx_{t,i_t}} 
+ \sum_{t: B_t=1} (\ell_{t,i_t}-m_{t,i_t})^2 \right].
\end{equation}
Since $x_{t, i_t} \geq \frac{x_{t-1, i_t}}{2}$ (see \pref{lem:stability}) and $x_{t-1, i_t} \geq \beta$ when $B_t=0$,
we have by the definition of $m_t$
\begin{align}
&\E\left[\sum_{t:B_t=0} \frac{(\ell_{t,i_t}-m_{t,i_t})^2}{Kx_{t,i_t}} \right]
\leq \E\left[\frac{2}{K\beta} \sum_{t:B_t=0} (\ell_{t,i_t}-m_{t,i_t})^2 \right]
= \E\left[\frac{2}{K\beta} \sum_{t:B_t=0} (\ell_{t,i_t}-\ell_{\rho_{i_t}(t-1),i_t})^2 \right] \notag \\ 
&\leq \E\left[\frac{2}{K\beta} \sum_{t=1}^T |\ell_{t,i_t}-\ell_{\rho_{i_t}(t-1),i_t}| \right] 
= \E\left[\frac{2}{K\beta}\sum_{i=1}^K \sum_{t: i_t=i} |\ell_{t, i}-\ell_{\rho_{i}(t-1),i}| \right]
\leq \frac{2V_1}{K\beta}. \label{eqn:entropy_stability}
\end{align}
On the other hand we have by the definition of $m_t$, $c_t$ and $\tau(t)$,
\begin{equation}\label{eqn:log_barrier_stability}
\E\left[\sum_{t: B_t=1} (\ell_{t,i_t}-m_{t,i_t})^2\right]
= \E\left[\sum_{t: B_t=1} \left(c_t - c_{\tau(t-1)}\right)^2\right]
= \E\left[\sum_{t=2}^T \left(c_{\tau(t)} - c_{\tau(t-1)}\right)^2\right].
\end{equation}
Next in \pref{lem:bias} we show an analogue of \pref{eqn:bias} which bounds 
$\E\left[\sum_{t=1}^T \inn{w_t-x_t, {\ell}_t}\right]$, the difference between playing $x_t$ and $w_t$, by
\begin{equation}\label{eqn:bias2}
\order(1) + \alpha \E\left[\sum_{t: B_{t+1}=1} |\ell_{t+1,i_{\tau(t)}}-\ell_{\tau(t),i_{\tau(t)}}| \right]
- \frac{\alpha}{2}\E\left[\sum_{t=2}^T(c_{\tau(t)}-c_{\tau(t-1)})^2 \right].
\end{equation}
It remains to bound the second term above.     
Note that for any integer $L$ between $1$ and $T$,
we have
\begin{align*}
&\E\left[\sum_{t: B_{t+1}=1} |\ell_{t+1,i_{\tau(t)}}-\ell_{\tau(t),i_{\tau(t)}}| \right] \\
&=  \E\left[\sum_{t: B_{t+1}=1, t+1-\tau(t) \geq L} |\ell_{t+1,i_{\tau(t)}}-\ell_{\tau(t),i_{\tau(t)}}| \right] +
\E\left[\sum_{t: B_{t+1}=1, t+1-\tau(t) < L} |\ell_{t+1,i_{\tau(t)}}-\ell_{\tau(t),i_{\tau(t)}}| \right] \\
&\leq \frac{T}{L} + \E\left[\sum_{t: B_{t}=1}\sum_{s=t}^{\min\{t+L-1, T\}} |\ell_{s+1,i_{t}}-\ell_{s,i_{t}}| \right] \\
&\leq \frac{T}{L} + \E\left[\sum_{t}\sum_{s=t}^{\min\{t+L-1, T\}}\sum_{i \in \calS_{t-1}} w_{t,i}|\ell_{s+1,i}-\ell_{s,i}| \right] \\
&\leq \frac{T}{L} + 4\beta\E\left[\sum_{t}\sum_{s=t}^{\min\{t+L-1, T\}}\sum_{i \in \calS_{t-1}} |\ell_{s+1,i}-\ell_{s,i}| \right]  \\
&\leq \frac{T}{L} + 4\beta L V_1.
\end{align*}
Here, the first inequality is by the fact that there are at most $T/L$ non-overlapping intervals with length at least $L$ (for the first term) and triangle inequality (for the second term);
the second inequality holds by taking the expectation of the indicator $B_t=1$ and the {\it obliviousness of the adversary} (the only place obliviousness is required);
and the third inequality holds since by \pref{lem:stability} $x_{t, i} \leq 2x_{t-1,i}$ which is at most $2\beta$ for all $i\in \calS_{t-1}$ and thus $w_{t,i} \leq x_{t,i} + \alpha_t \sum_{j\in\calS_{t-1}} x_{t,j} \leq 2\beta + 2K\alpha_t\beta \leq 4\beta$ by the condition $\alpha_t \leq  \alpha \leq \frac{1}{K}$.
Picking the optimal $L$ gives
\begin{equation}\label{eqn:term_with_T}
\E\left[\sum_{t: B_{t+1}=1} |\ell_{t+1,i_{\tau(t)}}-\ell_{\tau(t),i_{\tau(t)}}| \right]
= \order\left(\sqrt{\beta TV_1}\right).
\end{equation}
Finally, combining Eq.~\eqref{eqn:OMD_bound}, \eqref{eqn:entropy_stability}, \eqref{eqn:log_barrier_stability}, \eqref{eqn:bias2}, and \eqref{eqn:term_with_T},
and using $\alpha = 8\eta$ proves the theorem.
\end{proof}

\subsection{Technical lemmas}
\begin{lemma}[Multiplicative Stability]
\label{lem:stability}
If $\eta \leq \min\left\{\frac{1}{K}, \frac{1}{162}\right\}$,
\pref{line:OMD1} and \pref{line:OMD2} of \pref{alg:L1_norm} ensure $\max\left(\frac{x_{t+1,i}}{x_{t,i}}, \frac{x_{t,i}}{x_{t+1,i}}\right)\leq 2$ for all $t$ and $i$.
\end{lemma}

To prove this lemma we make use of the following auxiliary result,
where we use the notation $\norm{a}_M = \sqrt{a^\top M a}$ for a vector $a \in \fR^K$ and a positive semi-definite matrix $M \in \fR^{K\times K}$.
\begin{lemma}
\label{lemma: auxiliary}
For some arbitrary $b_1, b_2 \in \fR^K$, $a_0 \in \Delta_K$
and $\psi$ defined as in \pref{alg:L1_norm} with $\eta \leq \frac{1}{162}$,
define
\begin{align*}
\begin{cases}
a_1 = \argmin_{a\in \Delta_K} F_1(a), \quad \text{where\ } F_1(a)\triangleq \inn{a, b_1} + D_\psi(a, a_0), \\
a_2 = \argmin_{a\in \Delta_K} F_2(a), \quad \text{where\ } F_2(a)\triangleq \inn{a, b_2} + D_\psi(a, a_0). 
\end{cases}
\end{align*}
Then as long as $\|b_1-b_2\|_{\nabla^{-2} \psi(a_1)} \leq \frac{1}{9}$, we have for all $i\in[K]$, $\max\left\{\frac{a_{2,i}}{a_{1,i}}, \frac{a_{1,i}}{a_{2,i}}\right\}\leq \frac{27}{26}$. 
\end{lemma}
\begin{proof}[of Lemma~\ref{lemma: auxiliary}]
First, we prove $\|a_1-a_2\|_{\nabla^2 \psi(a_1)}\leq \frac{1}{3}$ by contradiction. Assume $\|a_1-a_2\|_{\nabla^2 \psi(a_1)} > \frac{1}{3}$. Then there exists some $a_2'$ lying in the line segment between $a_1$ and $a_2$ such that $\|a_1-a_2'\|_{\nabla^2 \psi(a_1)}=\frac{1}{3}$. By Taylor's theorem, there exists $\overline{a}$ that lies in the line segment between $a_1$ and $a_2'$ such that 

\begin{align}
F_2(a_2') 
&= F_2(a_1) + \inn{\nabla F_2(a_1), a_2'-a_1} + \frac{1}{2}\|a_2'-a_1\|_{\nabla^2 F_2(\overline{a})}^2 \nonumber \\
&= F_2(a_1) + \inn{b_2-b_1, a_2'-a_1} + \inn{\nabla F_1(a_1), a_2'-a_1} + \frac{1}{2}\|a_2'-a_1\|_{\nabla^2 \psi(\overline{a})}^2  \nonumber \\
&\geq F_2(a_1) - \|b_2-b_1\|_{\nabla^{-2} \psi(a_1)} \|a_2'-a_1\|_{\nabla^2 \psi(a_1)} + \frac{1}{2}\|a_2'-a_1\|_{\nabla^2 \psi(\overline{a})}^2 \nonumber \\
&\geq F_2(a_1) - \frac{1}{9}\times \frac{1}{3} + \frac{1}{2}\|a_2'-a_1\|_{\nabla^2 \psi(\overline{a})}^2     \label{eqn: to contradiction}
\end{align}
where in the first inequality we use H\"{o}lder inequality and the first-order optimality condition, 
and in the last inequality we use the conditions
$\|b_1-b_2\|_{\nabla^{-2} \psi(a_1)} \leq \frac{1}{9}$
and $\|a_1-a_2'\|_{\nabla^2 \psi(a_1)}=\frac{1}{3}$.  
Note that $\nabla^2 \psi(x)$ is a diagonal matrix
and $\nabla^2 \psi(x)_{ii} = \frac{1}{\eta}\frac{1}{x_i^2} + \frac{K}{\eta}\frac{1}{x_i} \geq  \frac{1}{\eta}\frac{1}{x_i^2}$. Therefore for any $i\in[K]$, 
\begin{align*}
\frac{1}{3} =  \|a_2'-a_1\|_{\nabla^2 \psi(a_1)}
\geq \sqrt{\sum_{j=1}^K \frac{(a_{2,j}'-a_{1,j})^2}{\eta a_{1,j}^2}}
\geq \frac{|a_{2,i}'-a_{1,i}|}{\sqrt{\eta} a_{1,i}}
\end{align*}
 and thus $\frac{|a_{2,i}'-a_{1,i}|}{ a_{1,i}} \leq \frac{\sqrt{\eta}}{3}\leq \frac{1}{27}$, which implies $\max\left\{ \frac{a_{2,i}'}{a_{1,i}}, \frac{a_{1,i}}{a_{2,i}'} \right\}\leq \frac{27}{26}$. Thus the last term in \pref{eqn: to contradiction} can be lower bounded by
\begin{align*}
 \|a_2'-a_1\|_{\nabla^2 \psi(\overline{a})}^2 
&= \frac{1}{\eta}\sum_{i=1}^K \left(\frac{1}{\overline{a}_i^2}+\frac{K}{\overline{a}_i} \right)(a_{2,i}'-a_{1,i})^2\geq  \frac{1}{\eta} \left(\frac{26}{27}\right)^2 \sum_{i=1}^K \left(\frac{1}{a_{1,i}^2}+\frac{K}{a_{1,i}} \right)(a_{2,i}'-a_{1,i})^2 \\
&\geq 0.9 \|a_2'-a_1\|_{\nabla^2 \psi(a_1)}^2 = 0.9\times \left(\frac{1}{3}\right)^2 = 0.1. 
\end{align*}
Combining with \pref{eqn: to contradiction} gives
\begin{align*}
F_2(a_2')\geq F_2(a_1) - \frac{1}{27} + \frac{1}{2}\times 0.1 > F_2(a_1). 
\end{align*}
Recall that $a_2'$ is a point in the line segment between $a_1$ and $a_2$. By the convexity of $F_2$, the above inequality implies $F_2(a_1)<F_2(a_2)$, contradicting the optimality of $a_2$. 

Thus we conclude $\|a_1-a_2\|_{\nabla^2 \psi(a_1)} \leq \frac{1}{3}$. Since $\|a_1-a_2\|_{\nabla^2 \psi(a_1)} \geq \frac{|a_{2,i}-a_{1,i}|}{\sqrt{\eta} a_{1,i}}$ for all $i$ according to previous discussions, we get $ \frac{|a_{2,i}-a_{1,i}|}{\sqrt{\eta} a_{1,i}} \leq \frac{\sqrt{\eta}}{3}\leq \frac{1}{27}$, which implies $\max\left\{\frac{a_{2,i}}{a_{1,i}}, \frac{a_{1,i}}{a_{2,i}}\right\}\leq \frac{27}{26}$. 
\end{proof}

\begin{proof}[of Lemma~\ref{lem:stability}]
We prove the following two stability inequalities
\begin{align}
\max\left\{\frac{x_{t,i}}{x_{t+1,i}'}, \frac{x_{t+1,i}'}{x_{t,i}}\right\}\leq \frac{27}{26} \label{eqn: to_prove1}\\
\max\left\{\frac{x_{t+1,i}}{x_{t+1,i}'}, \frac{x_{t+1,i}'}{x_{t+1,i}}\right\}\leq \frac{27}{26} \label{eqn: to_prove2}, 
\end{align}
which clearly implies the lemma since then $\max\left\{\frac{x_{t,i}}{x_{t+1,i}}, \frac{x_{t+1,i}}{x_{t,i}}\right\}\leq \frac{27}{26}\times \frac{27}{26}\leq 2$. To prove \pref{eqn: to_prove1}, observe that
\begin{align}
    \begin{cases}
          x_{t} = \argmin_{x\in \Delta_K} \inner{x, m_t} + D_\psi(x,x_t'),  \\
          x_{t+1}' =  \argmin_{x\in \Delta_K} \inn{x, \hat{\ell}_t} + D_\psi(x,x_t').    
    \end{cases}
    \label{eqn: update rule 1}
\end{align}
To apply Lemma~\ref{lemma: auxiliary} and obtain \pref{eqn: to_prove1}, we only need to show $\|\hat{\ell}_t-m_t\|_{\nabla^{-2}\psi(x_t)} \leq \frac{1}{9}$.  
Recall $\nabla^2 \psi(u)_{ii} = \frac{1}{\eta}\frac{1}{u_i^2} + \frac{K}{\eta}\frac{1}{u_i}$. By our algorithm we have
\begin{align*}
\|\hat{\ell}_t-m_t\|_{\nabla^{-2} \psi(x_t)}^2 
&\leq \sum_{i=1}^K \eta x_{t,i}^2 \left(\frac{(\ell_{t,i}-m_{t,i}) \one\{i_t=i\}}{w_{t,i}}\right)^2 \\
&\leq \sum_{i=1}^K \frac{\eta}{(1-\alpha)^2} x_{t,i}^2 \left(\frac{(\ell_{t,i}-m_{t,i}) \one\{i_t=i\}}{x_{t,i}}\right)^2  \tag{$\frac{x_{t,i}}{w_{t,i}}\leq \frac{1}{1-\alpha_t}\leq \frac{1}{1-\alpha}\ \forall i$} \\
&\leq \frac{\eta}{(1-\alpha)^2} \leq \frac{\frac{1}{162}}{(1-\frac{8}{162})^2} < \frac{1}{81},
\end{align*}
finishing the proof for \pref{eqn: to_prove1}.
To prove \pref{eqn: to_prove2}, we observe: 
\begin{align}
    \begin{cases}
          x_{t+1}' = \argmin_{x\in \Delta_K} D_\psi(x,x_{t+1}'), \\
          x_{t+1} =  \argmin_{x\in \Delta_K} \inner{x, m_{t+1}} + D_\psi(x,x_{t+1}'). 
    \end{cases}
    \label{eqn: update rule 2}
\end{align}
Similarly, with the help of Lemma~\ref{lemma: auxiliary}, we only need to show $\|m_{t+1}\|_{\nabla^{-2}\psi(x_{t+1}')}\leq \frac{1}{9}$. This can be seen by 
\begin{align*}
  \|m_{t+1}\|_{\nabla^{-2}\psi(x_{t+1}')}^2 \leq \sum_{i=1}^K \eta x_{t+1,i}'^2 m_{t+1,i}^2 \leq \eta \leq \frac{1}{81}. 
\end{align*}
This finishes the proof.
\end{proof}

\begin{lemma}
\label{lem:OMD}
If $\eta \leq \min\left\{\frac{1}{K}, \frac{1}{162}\right\}$,
\pref{line:OMD1} and \pref{line:OMD2} of \pref{alg:L1_norm} ensure for any arm $i^*$,
\[
\E\left[\sum_{t=1}^T \inn{x_t-\basis_{i^*}, {\ell}_t}\right]
\leq \mathcal{O}\left(\frac{K\ln T}{\eta}\right) + 
4\eta\E\left[\sum_{t:i_t\notin \calS_{t-1}} \frac{(\ell_{t,i_t}-m_{t,i_t})^2}{Kx_{t,i_t}} 
+ \sum_{t: i_t\in \calS_{t-1}} (\ell_{t,i_t}-m_{t,i_t})^2 \right].
\]
\end{lemma}

\begin{proof}
We will in fact prove a stronger statement
\[
\E\left[\sum_{t=1}^T \inn{x_t-\basis_{i^*}, {\ell}_t}\right]
\leq \mathcal{O}\left(\frac{K\ln T}{\eta}\right) + 
4\eta\E\left[\sum_{t=1}^T \min\left\{ \frac{(\ell_{t,i_t}-m_{t,i_t})^2}{Kx_{t,i_t}}, 
(\ell_{t,i_t}-m_{t,i_t})^2\right\} \right],
\]
which clearly implies the stated bound.

By standard analysis of optimistic mirror descent (e.g., Lemma 6 in \citep{wei2018more}, Lemma 5 in \citep{chiang2012online}), we have
\begin{align}
\inn{x_t - u, \hat{\ell}_t} \leq D_{\psi}(u, x_t') - D_{\psi}(u, x_{t+1}') + \inn{x_t-x_{t+1}', \hat{\ell}_t-m_t}  \label{eqn: instantaneous regret}
\end{align}
for all $u\in \Delta_K$. Specifically, we pick $u=\left(1-\frac{1}{T}\right)\basis_{i^*} + \frac{1}{KT}\one$. The first two terms on the right hand side of \pref{eqn: instantaneous regret} telescope when summing over $t$. The non-negative remaining term is 
\begin{align*}
D_\psi(u, x_1') &= \psi(u) - \psi(x_1') - \inn{\nabla \psi(x_1'), u - x_1'} \\
&= \psi(u) - \psi(x_1') \leq \frac{K\ln T}{\eta} + \frac{K\ln K}{\eta} = \order\left(\frac{K\ln T}{\eta}\right),  
\end{align*}
where in the first equality we use $x_1'=\frac{1}{K}\one$ and $\psi$'s definition. 
Below we focus on the last term in \pref{eqn: instantaneous regret}. 
Acoording to \pref{line:OMD1} and \pref{line:OMD2} of \pref{alg:L1_norm}, we can write
\begin{align}
    \begin{cases}
          x_{t} = \argmin_{x\in \Delta_K} F(x), \quad \text{where\ } F(x)\triangleq \inner{x, m_t} + D_\psi(x,x_t'),  \\
          x_{t+1}' =  \argmin_{x\in \Delta_K}F'(x), \quad \text{where\ } F'(x)\triangleq \inn{x, \hat{\ell}_t} + D_\psi(x,x_t').    
    \end{cases}
    \label{eqn: update rule 3}
\end{align}
By Taylor's theorem, there exists some $\overline{x}$ that lies in the line segment between $x_t$ and $x_{t+1}'$, such that 
\begin{align}
F'(x_t)-F'(x_{t+1}') = \inn{\nabla F'(x_{t+1}'), x_t-x_{t+1}'} + \frac{1}{2}\|x_t-x_{t+1}'\|_{\nabla^2 F'(\overline{x})}^2 \geq \frac{1}{2}\|x_t-x_{t+1}'\|_{\nabla^2 \psi(\overline{x})}^2, \label{eqn: stability_term_1}
\end{align}
where the last inequality uses the optimality of $x_{t+1}'$ and that $\nabla^2 F' = \nabla^2 \psi$. As shown in the proof of Lemma~\ref{lem:stability}, $\max\left\{\frac{x_{t+1,i}'}{x_{t,i}}, \frac{x_{t,i}}{x_{t+1,i}'}\right\}\leq \frac{27}{26}$, so 
\begin{align}
 \frac{1}{2}\|x_t-x_{t+1}'\|_{\nabla^2 \psi(\overline{x})}^2 
&= \frac{1}{2}\sum_{i=1}^K \frac{1}{\eta}\left(\frac{1}{\overline{x}_i^2} + \frac{K}{\overline{x}_i}\right)(x_{t,i}-x_{t+1,i}')^2  \nonumber \\
&\geq \frac{1}{2}\left(\frac{26}{27}\right)^2\sum_{i=1}^K \frac{1}{\eta}\left(\frac{1}{x_{t,i}^2} + \frac{K}{x_{t,i}}\right)(x_{t,i}-x_{t+1,i}')^2 \nonumber \\
&\geq \frac{0.9}{2}\|x_t-x_{t+1}'\|_{\nabla^2 \psi(x_t)}^2 \label{eqn: stability_term_2}
\end{align}
On the other hand, 
\begin{align}
F'(x_t)-F'(x_{t+1}') &= \inn{x_t-x_{t+1}', \hat{\ell}_t-m_t} + F(x_t) - F(x_{t+1}') \nonumber \\
&\leq  \inn{x_t-x_{t+1}', \hat{\ell}_t-m_t} \tag{by the optimality of $x_t$} \nonumber \\
&\leq \|x_t-x_{t+1}'\|_{\nabla^2 \psi(x_t)} \|\hat{\ell}_t-m_t\|_{\nabla^{-2} \psi(x_t)}.  \label{eqn: stability_term_3}
\end{align}
Combining \pref{eqn: stability_term_1}, \pref{eqn: stability_term_2} and \pref{eqn: stability_term_3} we get 
\begin{align*}
\|x_t-x_{t+1}'\|_{\nabla^2 \psi(x_t)} \leq \frac{2}{0.9} \|\hat{\ell}_t-m_t\|_{\nabla^{-2} \psi(x_t)}, 
\end{align*}
and thus
\begin{align*}
 &\inn{x_t-x_{t+1}', \hat{\ell}_t-m_t} \\
&\leq \|x_t-x_{t+1}'\|_{\nabla^2 \psi(x_t)} \|\hat{\ell}_t-m_t\|_{\nabla^{-2} \psi(x_t)} \leq 3  \|\hat{\ell}_t-m_t\|_{\nabla^{-2} \psi(x_t)}^2  \\
  &\leq 3\eta \sum_{i=1}^K \frac{1}{\frac{1}{x_{t,i}^2} + \frac{K}{x_{t,i}}} (\hat{\ell}_{t,i}-m_{t,i})^2 \\
  &\leq 3\eta \sum_{i=1}^K \min\left\{ x_{t,i}^2 ,  \frac{x_{t,i}}{K}  \right\}(\hat{\ell}_{t,i}-m_{t,i})^2 \\
  &\leq 3\frac{\eta}{(1-\alpha)^2}  \sum_{i=1}^K \min\left\{ x_{t,i}^2,  \frac{x_{t,i}}{K} \right\} \left(\frac{(\ell_{t,i}-m_{t,i})\one\{i_t=i\}}{x_{t,i}}\right)^2  \tag{$\frac{x_{t,i}}{w_{t,i}}\leq \frac{1}{1-\alpha_t}\leq \frac{1}{1-\alpha}\ \forall i$}  \\
  &\leq 4\eta \min\left\{ (\ell_{t,i_t}-m_{t,i_t})^2, \frac{\ell_{t,i_t}-m_{t,i_t})^2}{Kx_{t,i_t}}   \right\}  \tag{$\alpha=8\eta \leq \frac{8}{162}$}. 
\end{align*}
We have thus showed
\begin{align*}
\sum_{t=1}^T \inn{w_t-u, \hat{\ell}_t} \leq \order\left(\frac{K\ln T}{\eta}\right) + 4\eta \sum_{t=1}^T \min\left\{ (\ell_{t,i_t}-m_{t,i_t})^2, \frac{\ell_{t,i_t}-m_{t,i_t})^2}{Kx_{t,i_t}}   \right\}. 
\end{align*}
Finally realizing by $u$'s definition,
\begin{align*}
\sum_{t=1}^T \inn{u-\basis_{i^*}, \hat{\ell}_t} = \frac{1}{T}\sum_{t=1}^T \inner{-\basis_{i^*}+\frac{1}{K}\one, \hat{\ell}_t},
\end{align*}
combining the two inequalities above and taking expectation finish the proof. 
\end{proof}

\begin{lemma}
\label{lem:bias}
\pref{line:bias2} of \pref{alg:L1_norm} ensures
\begin{align*}
\E\left[\sum_{t=1}^T \inn{w_t-x_t, {\ell}_t}\right] 
\leq \order(1) + \alpha \E\left[\sum_{t: i_{t+1}\in\calS_t} |\ell_{t+1,i_{\tau(t)}}-\ell_{\tau(t),i_{\tau(t)}}| \right]
- \frac{\alpha}{2}\E\left[\sum_{t=2}^T(c_{\tau(t)}-c_{\tau(t-1)})^2 \right].
\end{align*}
\end{lemma}

\begin{proof}
First fix any $t > 1$ and denote $\tau(t-1)$ by $\tau$ for notational convenience
(note that $\tau$ is thus fixed at the beginning of round $t$).
Note that by the construction of $w_t$, 
we have $w_{t, i} = x_{t,i}$ for any $i \notin \calS_{t-1}\cup \{i_\tau\}$
and also $\sum_{i\in \calS_{t-1}\cup \{i_\tau\}}x_{t,i} 
= \sum_{i\in \calS_{t-1}\cup \{i_\tau\}}w_{t,i}$.
Therefore
\begin{align*}
\inn{w_t-x_t, {\ell}_t} 
&= \alpha_t\left(
\left(\sum_{i\in \calS_{t-1}\cup \{i_\tau\}}x_{t,i}\right) {\ell}_{t,i_\tau} 
-\sum_{i\in \calS_{t-1}\cup \{i_\tau\}}x_{t,i}{\ell}_{t,i}
\right) \\
&= \alpha_t\left( 
\left(\sum_{i\in \calS_{t-1}\cup \{i_\tau\}}w_{t,i}\right) {\ell}_{t,i_\tau} 
-\sum_{i\in \calS_{t-1}\cup \{i_\tau\}}x_{t,i}{\ell}_{t,i} 
\right) \\
&= \alpha_t\left( 
\sum_{i\in \calS_{t-1}\cup\{i_\tau\}}w_{t,i} {\ell}_{t,i_\tau} 
-\sum_{i\in  \calS_{t-1}\cup \{i_\tau\}}w_{t,i}{\ell}_{t,i} \right) 
+ \alpha_t \inn{w_t-x_t, {\ell}_t}\\
&= \alpha_t\left( 
\sum_{i\in  \calS_{t-1}}w_{t,i} {\ell}_{t,i_\tau} 
-\sum_{i\in  \calS_{t-1}}w_{t,i}{\ell}_{t,i} \right) 
+ \alpha_t \inn{w_t-x_t, {\ell}_t}.  
\end{align*}
Rearranging and using the definition of $\alpha_t$ gives:  
\begin{align*}
    \inn{w_t-x_t, {\ell}_t} 
    &= \frac{\alpha_t}{1-\alpha_t}\left(\sum_{i\in  \calS_{t-1}}w_{t,i} {\ell}_{t,i_\tau} -\sum_{i\in   \calS_{t-1}}w_{t,i}{\ell}_{t,i} \right) 
    = \alpha(1-c_\tau)\left( \sum_{i\in  \calS_{t-1}}w_{t,i} ({\ell}_{t,i_\tau} - {\ell}_{t,i}) \right).
\end{align*}
Taking expectation gives: 
\begin{align*}
    &\E[\inn{w_t-x_t, {\ell}_t}] 
    = \alpha\E\left[(1-c_\tau)\left( \sum_{i\in  \calS_{t-1}}w_{t,i}(\ell_{t,i_\tau}-\ell_{t,i}) \right)\right] \\
    & = \alpha \E\left[\one\{i_t \in \calS_{t-1}\}(1-c_\tau)(\ell_{t,i_\tau}-\ell_{t,i_t}) \right]  \\
    &= \alpha\E\left[ \one\{i_t \in \calS_{t-1}\}(1-c_\tau) (\ell_{t,i_\tau} - c_\tau+c_\tau -c_t)) \right] \\ 
    &\leq \alpha\E\left[\one\{i_t \in \calS_{t-1}\} |\ell_{t,i_\tau}-c_{\tau}| \right] + \alpha \E\left[\one\{i_t \in \calS_{t-1}\}(c_\tau-c_t - c_\tau^2 + c_t c_\tau ) \right] \\
    &= \alpha\E\left[\one\{i_t \in \calS_{t-1}\} |\ell_{t,i_\tau}-\ell_{\tau,i_\tau}|  \right] + \alpha \E\left[c_{\tau(t-1)} -c_{\tau(t)}- c_{\tau(t-1)}^2 + c_{\tau(t)} c_{\tau(t-1)}  \right],
\end{align*}
where in the last step we use the fact that $\tau(t)$ is $t$ if $i_t\in S_{t-1}$ and is $\tau(t-1)$ otherwise.
Finally summing over $t$ and telescoping finish the proof.
\end{proof}

\section{Proof of \pref{thm:chasing}}\label{app:chasing}
Since $m_{t+1}$ is the projection of $m_{t}$ on $\calK_{t+1}$ and also $m_{t+1}^\star \in \calK_{t+1}$,
by the generalized Pythagorean theorem we have
\[
\norm{m_{t+1}-m_{t+1}^\star}_2^2 + \norm{m_{t+1}-m_{t}}_2^2 \leq \norm{m_{t} - m_{t+1}^\star}_2^2.
\]
On the other hand, by triangle inequality we also have
\begin{align*}
\norm{m_{t} - m_{t+1}^\star}_2^2
&\leq \left(\norm{m_{t} - m_{t}^\star}_2 + \norm{m_{t}^\star - m_{t+1}^\star}_2\right)^2 \\
&= \norm{m_{t} - m_{t}^\star}_2^2 + 2 \norm{m_{t} - m_{t}^\star}_2\norm{m_{t}^\star - m_{t+1}^\star}_2 + \norm{m_{t}^\star - m_{t+1}^\star}_2^2 \\
&\leq \norm{m_{t} - m_{t}^\star}_2^2 + 6\norm{m_{t}^\star - m_{t+1}^\star}_2. \tag{$\calK_t \subset \calB$}
\end{align*}
Combining the two inequalities above, summing over $t$, and telescoping give
\[
\sum_{t=1}^{T-1} \norm{m_{t+1}-m_{t}}_2^2
\leq \norm{m_{1} - m_{1}^\star}_2^2 
+ 6\sum_{t=1}^{T-1} \norm{m_{t}^\star - m_{t+1}^\star}_2
\leq 4 + 6\sum_{t=1}^{T-1} \norm{m_{t}^\star - m_{t+1}^\star}_2,
\]
finishing the proof.

\end{document}